\documentclass{article}

\title{Ambiguity set and learning via Bregman and Wasserstein}
\author{Xin Guo, Johnny Hong, Nan Yang}
\date{}

\usepackage{graphicx} 
\usepackage{tikz-cd}
\usepackage{natbib}

\usepackage{algorithm}
\usepackage{algorithmic}

\usepackage{filecontents}

\usepackage{hyperref}

\usepackage{natbib}
\usepackage{graphicx}
\usepackage{float}
\usepackage[letterpaper, margin=0.8in]{geometry}

\usepackage{amsmath, amsfonts, dsfont, amssymb, mathtools, mathrsfs, amsthm} 
\usepackage{subcaption}
\newcommand{\EE}{\mathbb{E}} 
\newcommand{\X}{\mathcal{X}} 
\newcommand{\PP}{\mathbb{P}}
\newcommand{\QQ}{\mathbb{Q}} 
\newcommand{\RR}{\mathbb{R}}

\def\ed { \stackrel{d}{=} }
\def\convd { \stackrel{d}{\rightarrow} }

\newcommand{\ba}{\begin{array}}
\newcommand{\ea}{\end{array}}

\newcommand{\be}{\mathbb{E}}
\newcommand{\bp}{\mathbb{P}}

\newcommand{\XCal}{\mathcal{X}}

\,
\,

\newcommand{\bpm}{\begin{pmatrix}}
\newcommand{\epm}{\end{pmatrix}}
\newtheorem{definition}{Definition}[section]
\newtheorem{theorem}{Theorem}[section]
\newtheorem{corollary}{Corollary}[theorem]

\newtheorem{proposition}{Proposition}

\begin{document} 
\maketitle
\begin{abstract} 
Construction of ambiguity set in robust optimization relies on the choice of divergences between probability distributions.
In distribution learning, choosing appropriate probability distributions based on observed data is critical for approximating the true distribution.
To improve the performance of machine learning models, there has recently been interest in designing objective functions based on $L_p$-Wasserstein distance rather than the classical Kullback-Leibler (KL) divergence. In this paper, we derive concentration and asymptotic results using Bregman divergence. We propose a novel asymmetric statistical divergence called Wasserstein-Bregman divergence as a generalization of $L_2$-Wasserstein distance. We discuss how these results can be applied to the construction of ambiguity set in robust optimization.
\end{abstract} 

\section{Introduction}

Comparing probability distributions has been a recurring theme in many research areas of machine learning. 
In  distribution learning, for example, one is interested in approximating the true distribution by an element in a predetermined class of probability distributions, and this element is chosen based on the observed data.  Such choices rely on the divergence used in comparing distributions.
While there is an abundance in statistical divergences, there is no consensus about the ``ideal" way to measure the difference between distributions.

In the theory of robust optimization,  optimization problems are formulated under appropriate uncertainty sets for the model parameters and/or for the solutions against  a certain measure of robustness. For instance,  tractable uncertainty sets can be formulated in terms of chance constraints and expectation constraints under a given distribution $\PP$ \cite{Guan:2012ec}. However, when the distribution $\PP$ itself is unknown, which is the usual scenario in most data-driven research, the concept of ambiguity set is introduced \cite{Bayraksan:2015ge}. Thus, instead of optimizing under one particular distribution and  under a deterministic set, distributionally robust stochastic optimization, aka DRSO,  formulates optimization problems with a set of possible distributions, under the concept of ambiguity set. Specifically, one could consider minimizing the expected loss as follows,
\[
\min_{X\in\X}\max_{\PP\in\mathcal{P}}\EE_\PP [l(X;\xi)],
\]
where $X$ is the decision variable, allowed to vary inside the feasible region $\X$, and the random element $\xi$ follows distribution $\PP \in \mathcal{P}$, with $\mathcal{P}$ the ambiguity set and $l$ the loss function.

In the data driven setting where we have iid samples $\{\xi_i\}_{i=1}^n$ drawn from $\PP$, the ambiguity set $\mathcal{P}$ can be constructed so that it contains all distributions that are within a certain divergence from the empirical distribution, where the radius of the ambiguity set is large enough so that it contains $\PP$ with high probability. Alternative methods to construct ambiguity sets use moment constraints under $\PP \in \mathcal{P}$, where $\mathcal{P}$ consists of all probability distributions with first order and second order moments matching the sample moments. Again, the key is to define and measure the difference between various distributions.


In both the literature of learning and robust optimization, one popular choice to measure the difference between two distributions is the Kullback-Leibler divergence, which has strong theoretical foundation in information theory and large deviations \cite{Pardo:1999}. 
However, there are two issues in using the KL divergence. The first one is that the KL divergence between a continuous distribution and its empirical version, which is always a discrete distribution, is undefined (or infinite). The second issue is that KL divergence does not take into consideration the relative position of probability mass. As an example, consider the discrete distribution $\PP$ which puts $1/2$ mass on 0 and $1/2$ mass on 1, and the discrete distribution $\QQ$ which puts $1/2$ mass on $\epsilon$ and $1/2$ mass on $1-\epsilon$. The KL divergence does not reflect the convergence of $Q$ to $P$ as $\epsilon \downarrow 0$, hence it is too restrictive. It is therefore natural to use alternative measures for distributions, such as $f$-divergence, $L_p$-Wasserstein distance, and Prohorov metric. (See Section \ref{RelatedWork} for more details).

On the other hand, KL divergence belongs to a class of divergences known as \textit{Bregman divergences}. Bregman divergences \cite{Bregman:1967ab} are introduced by Lev Bregman in 1967 in solving a problem in convex optimization. Since its inception, Bregman divergences have found applications not only in convex optimization but also in statistics and machine learning, for example, clustering \cite{Lucic:2016uv} \cite{Banerjee:2005vsa}, inverse problems \cite{LeBesenerais:1999} \cite{Jones:1990}, classification \cite{Srivastava:2007}, logistic regression and AdaBoost \cite{Collins:2002wc} \cite{Murata:2004} \cite{Lafferty:1999}, regression \cite{Kivinen:2001}, mirror descent (\cite{Nemirovski:1983ab}), and generalized accelerated descent algorithms \cite{Wibisono:2016gx} \cite{Taskar:2006}. Bregman divergences are asymmetric in general, which can potentially be more desirable in the setting of comparing distributions, compared to a symmetric measure such as $L_p$-Wasserstein distance. 

Our goal is to address the following questions:

\begin{itemize}
\item How can we define appropriate divergences in the general setting of comparing distributions?
\item How can we define  appropriate divergences, in the particular context of robust optimization and distribution learning?
\end{itemize}

In this paper, we report some progress toward our goal. Our main contributions are as follows:
\begin{itemize}
\item We derive a weak convergence result using Bregman divergence in parametric distributions. The result describes precisely how the Hessian of the underlying convex function in Bregman divergence impacts the statistical properties of the divergence measure in the asymptotic setting.
\item In the non-asymptotic setting, we prove concentration results using Bregman divergence between the true discrete distribution and the empirical distributions. This allows the construction of ambiguity set in robust optimization.
\item We propose a novel statistical divergence called \textit{Wasserstein-Bregman divergence}, which is essentially a marriage between Wasserstein distance and Bregman divergence. We find that this divergence has the ability to capture the asymmetry in comparing distributions, while  retains nice analytical properties of Wasserstein distance for the purpose of optimization.
\end{itemize}

\subsection{Related Work}
\label{RelatedWork}
\paragraph{DRSO with KL Divergence.}
In \cite{Hu:2013tc}, they formulate a robust optimization problem in terms of a KL divergence constraint and show that the problem can be converted into a convex optimization problem which can be solved analytically. In
\cite{Guan:2012ec}, they show that chance constraints with KL divergence ambiguity sets can be reformulated into a traditional chance constraint problem with different risk levels.

\paragraph{DRSO with $L_p$-Wasserstein Distance.}
In \cite{Esfahani:2015wh}, they propose the use of $L_1$-Wasserstein ambiguity set. They show that Wassstein ambiguity sets provide a better out-of-sample guarantee than the KL divergence, because a continuous $\PP$ will always be outside the KL divergence ball centered at the empirical distribution $\hat{\PP}_n$, which is discrete, whereas the Wasserstein ball contains continuous as well as discrete distributions. They also show that the robust optimization problem, under some mild conditions, can be converted into a finite-dimensional convex programming problem, solvable in polynomial time. In \cite{Shafieezadeh-Abadeh:2015wr}, they use Wasserstein ambiguity set for distributionally robust logistic regression. Specifically they study $\inf_\beta \sup_{\PP \in \mathcal{P}} \EE_{\PP}[l_\beta(x,y)]$, where $l_\beta(x,y)$ is the logloss function with parameter $\beta$. They show that this problem has a tractable convex reformulation and provide confidence interval for the objective function, which is the out of sample performance guarantee. In \cite{Anonymous:W0HO4cob}, they use the $L_1$-Wasserstein ball as the ambiguity set. They show that the candidate probability distributions in the ball can be reduced to a subset whose elements can be described using extreme/exposed points of the set, hence a tractable reformulation of the original problem becomes possible.
In \cite{Gao:2016vo}, they consider the $L_p$-Wasserstein ball for $p \geq 1$, and give necessary and sufficient conditions for the worst-case distributions to exist. In \cite{Fournier:2015kk}, they inspect the convergence rate of the empirical distribution to the true distribution under Wasserstein distance.

\paragraph{Distribution Learning with $L_2$-Wasserstein Distance.}
In \cite{Arjovsky:2017vh}, they use neural network to learn probability density and define the objective function for optimization to be the $L_2$-Wasserstein distance. They have shown promising results on a numerical experiments in image generation.

\section{Background}
In this section, we will review definitions and relevant properties of Bregman divergence and Wasserstein distance.

\subsection{Bregman Divergence}
\begin{definition}
For two vectors $x$ and $y$ in $\RR^d$ and a strictly convex function $\phi(x): \RR^d \to \RR$, the \textit{Bregman divergence} is defined as
\[
D_\phi(x,y) = \phi(x) - \phi(y) - \langle \nabla \phi(y), x-y\rangle.
\]
\end{definition}

For two continuous distributions $\PP$ and $\QQ$, one can define Bregman divergence as in \cite{Jones:1990tf},
\begin{align*}
&D_\phi(\PP,\QQ) \\
&= \int\left[\phi(p(x)) - \phi(q(x)) - \phi'(q(x))(p(x)-q(x))\right] d\mu(x),
\end{align*}
where $p(x)$ and $q(x)$ are probability density functions of $\PP$ and $\QQ$ respectively, $\mu$ is the base measure, and $\phi: \RR \to \RR$ is a strictly convex function.

Examples of Bregman divergences include
\begin{itemize}
	\item $L^2$ loss: $D_\phi(x,y) = \|x-y\|_2^2$, where $\phi(x) = \|x\|_2^2$,
	\item Itakura-Saito divergence: $D_\phi(x,y) = x/y - \log(x/y) - 1$, where $\phi(x) = -\log x$,
	\item KL divergence: $D_\phi(x,y) = \sum_{i = 1}^d x_i \log (x_i/y_i)$, where $\phi(x) = \sum_{i=1}^d x_i \log x_i$,
	\item Mahalanobis distance: $D_\phi(x,y) = (x-y)^T A (x-y)$, where $\phi(x) = x^T A x$, $A$ is a strictly positive definite matrix.
\end{itemize}
As a divergence function, $D_\phi(x,y)$ is always nonnegative by the convexity of $\phi$. $D_\phi(x,y) = 0$ if and only if $x = y$. However, it is not a metric because it is not symmetric, and it does not satisfy the triangle inequality. In \cite{Pardo:2003fr}, they show an asymptotic equivalence between $f$-divergences (in particular, $\chi^2$-divergence) and Bregman divergences under some conditions.

\paragraph{$k$-means Using Bregman.}
In \cite{Banerjee:2005jd}, they show that conditional expectation is the optimal predictor for all Bregman divergences. Moreover, Bregman divergences are the only class of such loss functions. This property ensures the convergence of $k$-means algorithm when Bregman divergence is used as a loss function. 

\paragraph{Connections with Exponential Family.}
In \cite{Banerjee:2005vsa}, they show that there is a one-to-one correspondence between Bregman divergences and exponential family. That is, take an exponential family in a canonical form of:
\[
p_\theta(x) = \exp(\theta^Tx - \psi(\theta))h(x),
\]
where $\theta, x\in\RR^d$. $\psi$ is the cumulant function with its Legendre convex conjugate $\phi$ defined as
\[
\phi(x) = \sup_t [\langle x, t \rangle - \psi(t)].
\]
Then
\[
p_\theta(x) = \exp(-D_\phi(x,\mu(\theta)) - g_\phi(x)),
\]
with $\mu(\theta) = \nabla \psi(\theta)$. This one-to-one correspondence comes from the duality property of Bregman divergence, which states that
\[
D_\phi(p,q) = D_\phi(q^*, p^*),
\]
with $p^* = \nabla \phi(p)$ and $q^* = \nabla \phi(q)$.

\paragraph{Connections to Fisher Information.}
In the case where $X\sim p_\theta$ and $p_\theta$ belongs to a regular exponential family, the Fisher information of $\mu = \EE X$ has a nice representation. For notation simplicity, we present the result for the one-dimensional case. This result can be easily extended to higher dimensions.

\begin{proposition} \label{FisherInfo}
Suppose $X\sim p_{\theta}$ belongs to a regular exponential family. Let $\mu = \EE(X)$, $\psi$ be the cumulant function and $\phi$ be the convex conjugate of $\psi$. Assume that $\psi$ is three-time differentiable. Then
$$I(\mu) = \EE\left[\frac{\partial^2}{\partial \mu^2} D_{\phi}(x, \mu)\right] = \phi''(\mu).$$
\end{proposition}

\begin{proof}
The first equality follows directly from the representation $p_\theta(x) = \exp(-D_{\phi}(x, \mu) - g_{\phi}(x))$. The second equality follows from a straightforward calculation,
\begin{equation*}
\begin{split}
\EE\left[\frac{\partial^2}{\partial \mu^2} D_{\phi}(x, \mu)\right] &= \EE\left[\frac{\partial^2}{\partial \mu^2} [\phi(x) - \phi(\mu) - \phi'(\mu)(x - \mu)]\right] \\
&= \EE\left[\frac{\partial}{\partial \mu}[-\phi''(\mu)(x - \mu)]\right]\\
&= \EE[-\phi'''(\mu)(x-\mu) + \phi''(\mu)]\\
&=\phi''(\mu).
\end{split}
\end{equation*}
\end{proof}

\paragraph{Bias-Variance Decomposition.}

In \cite{Buja05lossfunctions}, they show that expected Bregman divergence has a bias-variance decomposition
$$\EE D_\phi (\hat{\theta}, \theta) = D_\phi(\EE \hat{\theta}, \theta) + \EE D_\phi(\hat{\theta}, \EE \hat{\theta}).$$ 

Setting $\phi(x) = \| x \|_2^2$ recovers the usual bias-variance decomposition for squared-error loss,

\begin{align*}
&\EE D_\phi (\hat{\theta}, \theta) \\
&= \EE[(\hat{\theta} - \theta)^2] \\
&= (\EE \hat{\theta} - \theta)^2 + \EE[(\hat{\theta} - \theta)^2] \\
&= D_\phi(\EE \hat{\theta}, \theta) + \EE D_\phi(\hat{\theta}, \EE \hat{\theta}).
\end{align*}

Figure \ref{fig:kNN_tradeoff} shows how various choices can lead to different measures of bias-variance tradeoff in selecting the number of neighbors for $k$-nearest neighbor ($k$-NN) algorithm.

\begin{figure}[H]
\centering
\includegraphics[width=0.37\textwidth]{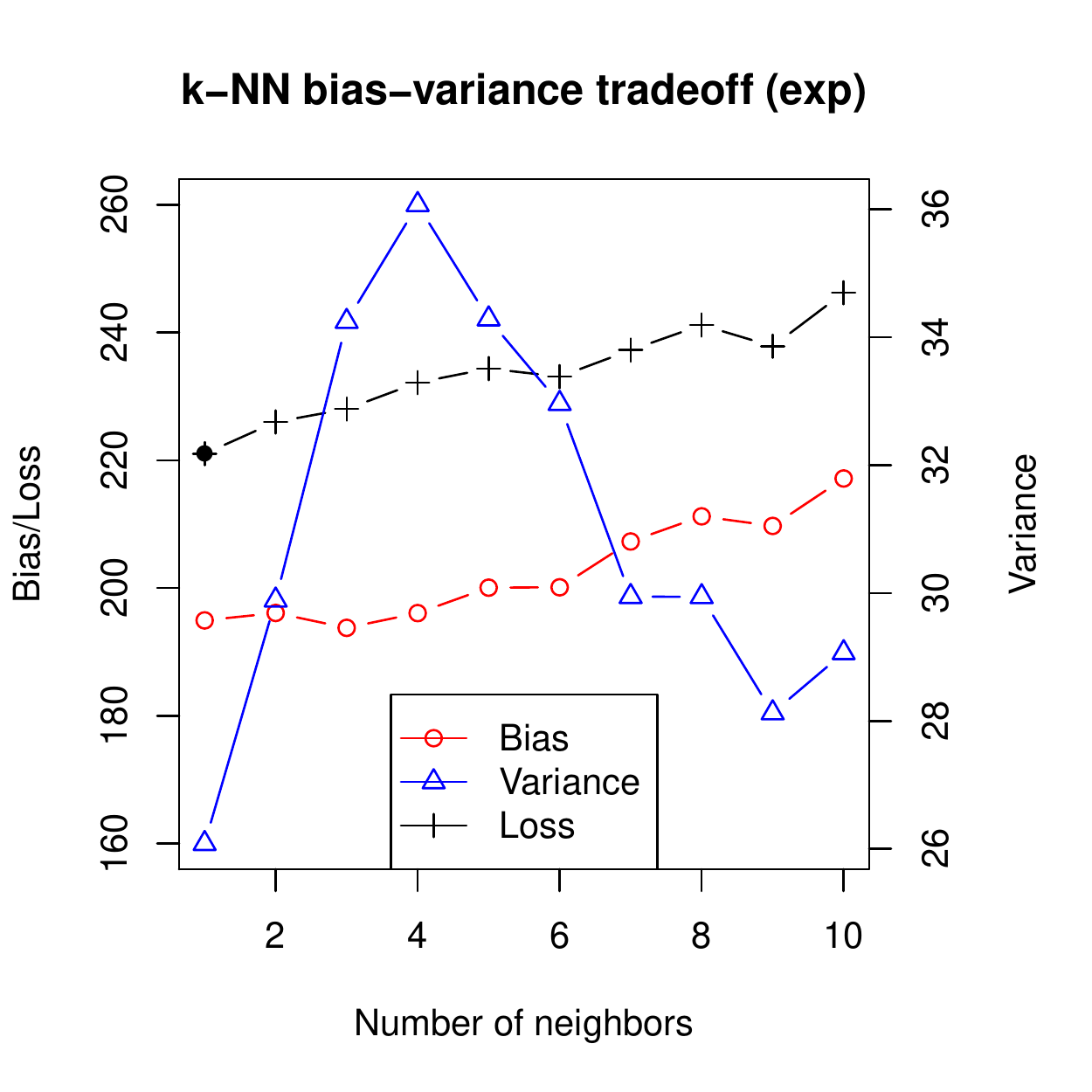}
\includegraphics[width=0.37\textwidth]{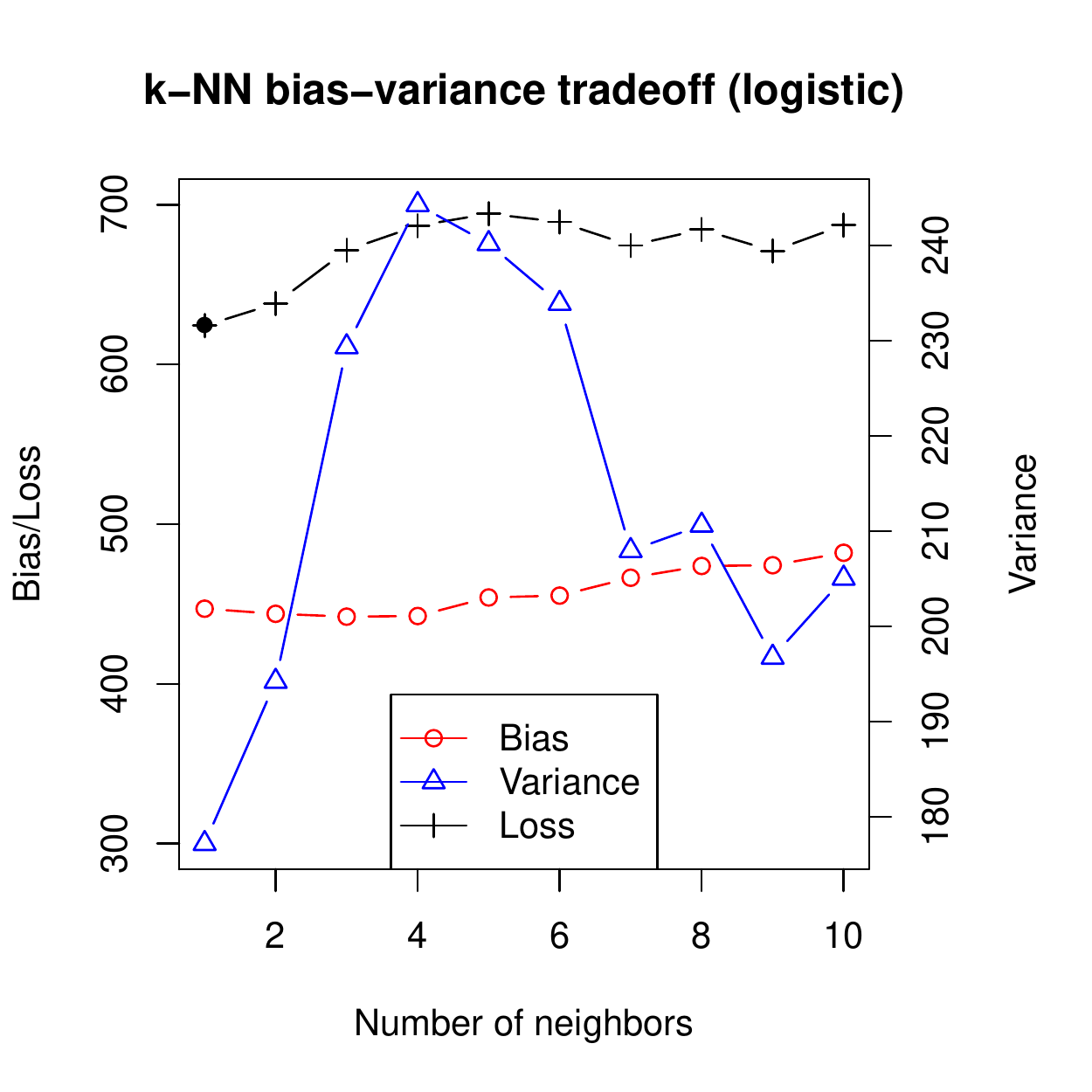}
    \caption{The bias-variance tradeoff of $k$-NN algorithm based on the loss functions $D_\phi(x, y) = e^x - e^y - e^y(x - y)$ and $D_\phi(x, y) = x \log \left(\frac{x}{y}\right) + (1 - x)\log\left(\frac{1 - x}{1 - y}\right)$. For each plot, the solid dot indicates the parameter that minimizes the corresponding loss function. The data used for this illustration is the \texttt{spam} dataset collected at Hewlett-Packard Labs, readily available in the R package \textit{kernlab}. }
    \label{fig:kNN_tradeoff}
\end{figure}

\subsection{Wasserstein Distance}

Wasserstein distance is a divergence defined between probability distributions on a given metric space. It is also known as Kantorovich-Monge-Rubinstein metric. Wasserstein distances are vastly used in optimal transport \cite{Villani:2009ab}, and they have found applications in many areas, such as the study of mixing for Markov chains in probability theory \cite{Dobrusin310:1996} \cite{Peres679:un}, rates of fluctuations for empirical measures in statistics \cite{Rachev695:1991} \cite{Rachev696:1998} \cite{Dobric307:1995}, and propagation of chaos in statistical mechanics \cite{Dobrusin308:1970} \cite{Spohn757:1991}.

\begin{definition}
For any two probability distributions $\PP$ and $\QQ$ defined on a compact metric space $(\mathcal{X}, d)$, the \textit{Wasserstein distance} between $\PP$ and $\QQ$ is defined as
\begin{align*}
	W_p(\PP, \QQ) &= \left( \inf_{\gamma \in \Pi(\PP,\QQ)} \int_{\mathcal{X} \times \mathcal{X}} d(x, y)^p d\gamma(x,y)\right)^{1/p},
\end{align*}
where $\Pi(\PP,\QQ)$ denotes the set of all couplings of $\PP$ and $\QQ$, i.e., all joint distributions defined on $\X \times \X$ with marginal distributions being $\PP$ and $\QQ$.
\end{definition}

For example, if $\mathcal{X} \subset \RR^d$, a natural choice of metric is $d(x, y) = \| x - y \|_p$. This leads to the definition of $L^p$-Wasserstein distance.

\begin{definition}
For any two probability distributions $\PP$ and $\QQ$ defined on a compact metric space $\mathcal{X} \subset \RR^d$, the \textit{Wasserstein distance of order $p\in[1,\infty]$} (or $L_p$-Wasserstein distance) between $\PP$ and $\QQ$ is defined as
\begin{align*}
	W_p(\PP, \QQ) &= \left( \inf_{\gamma \in \Pi(\PP,\QQ)} \int_{\mathcal{X} \times \mathcal{X}} \|x - y\|_p^p d\gamma(x,y)\right)^{1/p},
\end{align*}
where  $\Pi(\PP,\QQ)$ denotes the set of all couplings of $\PP$ and $\QQ$, i.e., all joint distributions defined on $\X \times \X$ with marginal distributions being $\PP$ and $\QQ$.
\end{definition}

For probability distributions, convergence under Wasserstein distance of order $p$ is equivalent to weak convergence plus convergence of the first $p$ moments.

For example, the Euclidean metric leads to the $L_2$-Wasserstein distance. The squared $L_2$-Wasserstein distance is defined as
\begin{align*}
	W_2(\PP, \QQ)^2 &= \inf_{\gamma \in \Pi(\PP,\QQ)} \int_{\mathcal{X} \times \mathcal{X}} \|x - y\|^2 d\gamma(x,y) \\
	&= \inf_{\gamma \in \Pi(\PP,\QQ)} \EE_{X,Y \sim \gamma} [\|X - Y\|_{L^2}] \\
	&= \EE_{X \sim \PP} [\|X\|_2^2] + \EE_{Y \sim \QQ} [\|Y\|_2^2] \\
	&\qquad+ \inf_{\gamma \in \Pi(\PP,\QQ} [\langle -2Y, X \rangle].
\end{align*}
Wasserstein distance is a metric, so it is always nonnegative, it is equal to 0 if and only if $\PP = \QQ$. It is symmetric, and it satisfies the triangle inequality.

\section{Main Results}
Let $p = (p_1, \ldots, p_d) \in \RR^d$ be the probability distribution of a discrete random variable $X$, where $p_i = \PP(X = a_k$, $k \in \{1,2,\ldots, d\}$. Let $\hat{p}_n = (\hat{p}_{n,1}, \ldots, \hat{p}_{n,d}) \in \RR^d$ be the random vector denoting the empirical distribution of a sequence of iid random variables $\{X_i\}_{i=1}^n$, where each $X_i$ has the same distribution as $X$. That is,
\[
	\hat{p}_{n} = \left(\frac{1}{n}\sum_{i=1}^n 1\{X_i = a_1\}, \ldots, \frac{1}{n}\sum_{i=1}^n 1\{X_i = a_d\}\right).
\]

\subsection{Concentration of Bregman Divergence}

We first establish that the Bregman divergence $D_\phi$ between the empirical distribution and the true distribution concentrates around the mean, where the rate can be expressed in terms of the gradient of the convex function $\phi$.

\begin{theorem} \label{nonconvexConcentration}
Consider the random variable $Z = D_\phi(\hat{p}_n, p)$, the Bregman divergence between $\hat{p}_n$ and $p$,
\[
Z = D_\phi(\hat{p}_{n}, p) = \phi(\hat{p}_{n}) - \phi(p) - \langle \nabla\phi(p), \hat{p}_{n} - p \rangle,
\]
where $\phi: [0,1]^d \to \RR$ is a strictly convex function. Then the following concentration inequality holds for all $\epsilon > 0$:
\[
	\PP\{Z - \EE [Z] \ge \epsilon \} \le \exp(\frac{-n^2\epsilon^2}{4dM_\phi}\}),
\]
where $M_\phi = \max_{t \in \Delta^{d-1}} \|\nabla\phi(t)\|_2$, and $\Delta^{d-1}$ is the standard $(d-1)$-simplex, which is the set $\{(t_1,t_2, \ldots, t_d) \in \RR^d| \sum_{i=1}^d t_i = 1, t_i \ge 0, \forall i\}$.
\end{theorem}
\begin{proof}
Let $(X_1, \ldots, X_{i-1}, X_i, X_{i+1}, \ldots, X_n)$ be iid random variables from distribution $p$. Define another sequence of random variables $(X_1, \ldots, X_{i-1}, X_i', X_{i+1}, \ldots, X_n)$, in which only the $i$-th element in the sequence is different. Let the corresponding empirical distribution be $\hat{p}_n'$. Then
\[
Z' = D_\phi(\hat{p}_{n}', p) = \phi(\hat{p}_{n}') - \phi(p) - \langle \nabla\phi(p), \hat{p}_{n}' - p \rangle.
\]
The difference of $Z$ and $Z'$ is
\begin{align*}
	Z' - Z = \phi(\hat{p}_{n}') - \phi(\hat{p}_{n}) + \langle \nabla\phi(p), \hat{p}_{n} - \hat{p}_n' \rangle.
\end{align*}
Notice that by construction, $\hat{p}_{n} - \hat{p}_n'$ is a vector with an element being $1/n$, an element being $-1/n$, and all other elements being zeros. Therefore by the Cauchy-Schwarz inequality,
\begin{align*}
|\langle \nabla\phi(p), \hat{p}_{n} - \hat{p}_n' \rangle| \le \|\nabla \phi(p)\|_2 \|\hat{p}_{n} - \hat{p}_n'\|_2 \\
= \frac{\sqrt{2}}{n} \|\nabla \phi(p)\|_2 \le \frac{\sqrt{2}}{n} M_\phi.
\end{align*}
Also by the Taylor's expansion,
\begin{align*}
\phi(\hat{p}_{n}') - \phi(\hat{p}_{n}) &= \|\nabla\phi(\xi)\|_2 \|\hat{p}_{n} - \hat{p}_n'\|_2 \\
&\le M_\phi \|\hat{p}_{n} - \hat{p}_n'\|_2 \\
&= \frac{\sqrt{2}}{n} M_\phi,
\end{align*}
where $\xi$ is a random vector which is a convex combination of $\hat{p}_{n}$ and $\hat{p}_{n}'$. Therefore by the triangle inequality,
\begin{align*}
	|Z - Z'| &\le |\langle \nabla\phi(p), \hat{p}_{n} - \hat{p}_n' \rangle| + |\phi(\hat{p}_{n}') - \phi(\hat{p}_{n})| \\
	&\le\frac{2\sqrt{2}}{n} M_\phi.
\end{align*}
Hence by the bounded difference inequality \cite{Talagrand772:1995},
\begin{align*}
	\PP\{Z - \EE [Z] \ge \epsilon \} &\le  \exp(\frac{-n^2\epsilon^2}{4dM_\phi}).
\end{align*}
\end{proof}

Notice that Bregman divergence is only convex with respect to its first argument, which in the previous case is $\hat{p}_n$. To construct a convex ambiguity region, we need to reverse the order of $\hat{p}$ and $p$ to make the unknown true distribution the first argument. Hence we also prove the following concentration inequality:

\begin{theorem} \label{convexConcentration}
Consider the random variable $Y = D_\phi(p,\hat{p}_n)$, the Bregman divergence between $p$ and $\hat{p}_n$:
\[
Y = D_\phi(p, \hat{p}_{n}) = \phi(p) - \phi(\hat{p}_{n}) - \langle \nabla\phi(\hat{p}_{n}), p - \hat{p}_{n}\rangle,
\]
where $\phi: [0,1]^d \to \RR$ is a strictly convex function. Then we have the following concentration inequality for all $\epsilon > 0$:
\[
	\PP(Y - \EE Y \ge \epsilon) \le \exp(-\frac{n^2\epsilon^2}{4d(M_\phi + L_\phi)^2}),
\]
where $L_\phi$ is the Lipschitz constant of $\nabla\phi$, and $M_\phi = \max_{t \in \Delta^{d-1}} \|\nabla\phi(t)\|$. $\Delta^{d-1}$ is the standard $(d-1)$-simplex, which is the set $\{(t_1,t_2, \ldots, t_d) \in \RR^d| \sum_{i=1}^d t_i = 1, t_i \ge 0, \forall i\}$.
\end{theorem}
\begin{proof}

Let $(X_1, \ldots, X_{i-1}, X_i, X_{i+1}, \ldots, X_n)$ be iid random variables from distribution $p$. Define another sequence of random variables $(X_1, \ldots, X_{i-1}, X_i', X_{i+1}, \ldots, X_n)$, in which only the $i$-th element in the sequence is different. Let the corresponding emprical distirbution be $\hat{p}_n'$. Then
\[
Y' = D_\phi(p, \hat{p}_{n}') = \phi(p) - \phi(\hat{p}_{n}') - \langle \nabla\phi(\hat{p}_{n}'), p - \hat{p}_{n}'\rangle.
\]
The difference of $Y$ and $Y'$ is
\begin{align*}
	Y' - Y &= \phi(\hat{p}_{n}) - \phi(\hat{p}_{n}') \\
	&\qquad+ \langle \nabla\phi(\hat{p}_n), p - \hat{p}_{n} \rangle - \langle \nabla\phi(\hat{p}_n'), p - \hat{p}_{n}' \rangle.
\end{align*}
By the proof of Theorem \ref{nonconvexConcentration},
\[
	\phi(\hat{p}_{n}) - \phi(\hat{p}_{n}') \le \frac{\sqrt{2}}{n}M_\phi.
\]
Meanwhile
\begin{align*}
	\langle \nabla\phi(\hat{p}_n), p - &\hat{p}_{n} \rangle - \langle \nabla\phi(\hat{p}_n'), p - \hat{p}_{n}' \rangle \\
	&= \langle \nabla\phi(\hat{p}_n) - \nabla \phi(\hat{p}_n'), p \rangle \\
	&\qquad- \langle \nabla \phi(\hat{p}_n), \hat{p}_n \rangle + \langle \nabla \phi(\hat{p}_n'), \hat{p}_n' \rangle.
\end{align*}
Since $\nabla \phi$ is defined on the compact region $[0,1]^d$, we can assume without loss of generality that it has Lipschitz constant $L_\phi$. Then by the Cauchy-Schwarz inequality,
\begin{align*}
|\langle \nabla\phi(\hat{p}_n) - \nabla \phi(\hat{p}_n'), p \rangle| &\le \|p\|_2 \|\nabla\phi(\hat{p}_n) - \nabla\phi(\hat{p}_n')\|_2 \\
&\le L_\phi \|\hat{p}_n - \hat{p}_n'\|_2 = \frac{\sqrt{2}}{n}L_\phi,
\end{align*}
and similarly
\begin{align*}
	&\qquad|- \langle \nabla \phi(\hat{p}_n), \hat{p}_n \rangle + \langle \nabla \phi(\hat{p}_n'), \hat{p}_n' \rangle| \\
	&= |\langle \nabla \phi(\hat{p}_n), \hat{p}_n' - \hat{p}_n \rangle + \langle \nabla \phi(\hat{p}_n') - \nabla \phi(\hat{p}_n), \hat{p}_n' \rangle| \\
	&\le |\langle \nabla \phi(\hat{p}_n), \hat{p}_n' - \hat{p}_n \rangle| + |\langle\nabla \phi(\hat{p}_n') - \nabla \phi(\hat{p}_n), \hat{p}_n' \rangle| \\
	&\le \frac{\sqrt{2}}{n} M_{\phi} + \frac{\sqrt{2}}{n}L_\phi.
\end{align*}
Therefore
\[
|Y' - Y| \le 2\left(\frac{\sqrt{2}}{n} M_{\phi} + \frac{\sqrt{2}}{n}L_\phi\right).
\]
By the bounded difference inequality,
\[
\PP(Y - \EE Y \ge \epsilon) \le \exp(-\frac{n^2\epsilon^2}{4d(M_\phi + L_\phi)^2}).
\]
\end{proof}

\subsection{Weak Convergence of Bregman Divergence}
In this section, we will show that in the asymptotic case, Bregman divergence between the true parameters of a distribution and the maximum likelihood estimator of the parameters will converge in distribution to a finite weighted sum of independent $\chi^2$ distributed random variables. This result allows us to construct asymptotic ambiguity sets according to the quantiles of the asymptotic distribution.
\begin{theorem}
 	Suppose there exists a family of probability distributions $\PP_\theta$ parametrized by $\theta\in\Theta\subset\RR^d$. Suppose we have iid data $\{X_i\}_{i=1}^n$, and $\hat{\theta}_n$ is the maximum likelihood estimator of $\theta$. Then
 	\[
\lim_{n \to \infty} n D_{\phi}(\theta, \hat{\theta}_n) \convd \frac{1}{2} \sum_{i = 1}^r \beta_i Z_i^2,
	\]
	where $Z_i$'s are independent standard Gaussian random variables, $D_\phi$ denotes the Bregman divergence characterized by $\phi$, $\beta_i$'s are the non-zero eigenvalues of the matrix $H\Sigma$ and $r = rank(\Sigma^T H \Sigma)$, with $H$ the Hessian of $\phi$ at $\theta$ and $\Sigma$ the inverse Fisher information matrix.
\end{theorem}

\begin{proof}
First, write the Taylor expansion of $\phi$ around $\hat{\theta}_n$,
\begin{align*}
	\phi(\theta) &= \phi(\hat{\theta}_n) + \langle \theta - \hat{\theta}_n, \nabla \phi(\hat{\theta}_n) \rangle \\
	&\qquad+ \frac{1}{2}(\theta - \hat{\theta}_n)^T H(\hat{\theta}_n) (\theta - \hat{\theta}_n) + o(\|\theta - \hat{\theta}_n\|_2^2),
\end{align*}
where $H(\hat{\theta})$ is the Hessian of $\phi(x)$ at $x = \hat{\theta}$. Notice that by the properties of maximum likelihood estimators, as $n \to \infty$,
\[
\sqrt{n}(\theta - \hat{\theta}_n) \convd N(0,\mathcal{I}^{-1}) \ed N(0,\Sigma),
\]
where 
\[
(\mathcal{I})_{ij}= -\EE \frac{\partial ^2 \log L}{\partial \theta_i \partial \theta_j}
\]
is the Fisher information matrix of the underlying true distribution, with $L$ being the likelihood function. Also,
\[
H(\hat{\theta}_n) \to H(\theta)
\]
in probability, and
\[
n \cdot o(\|\theta - \hat{\theta}_n\|_2^2) \to 0 
\]
in probability. Therefore by the Slutsky's theorem,
\begin{align*}
	nD_{\phi}(\theta, \hat{\theta}_n) &= n(\phi(\theta) - \phi(\hat{\theta}_n) - \langle \theta - \hat{\theta}_n, \nabla \phi(\hat{\theta}_n) \rangle) \\
	&= \frac{1}{2}\sqrt{n}(\theta - \hat{\theta}_n)^T H \sqrt{n}(\theta - \hat{\theta}_n) \\
	&\qquad+ n\cdot o(\|\theta - \hat{\theta}_n\|_2^2) \\
	&\convd \frac{1}{2} X^T H X,
\end{align*}
where $X \ed N(0,\Sigma)$. Let $S \in \RR^{d\times s}$ be a square root of $\Sigma$. Since $\Sigma$ and $H$ are positive semidefinite, by spectral theorem, we can write $S^THS = R^T\Lambda R$, where $\Lambda = diag(\beta_1, \ldots, \beta_r)$, which is the diagonal matrix of non-zero eigenvalues of $S^THS$, hence is also the diagonal matrix of non-zero eigenvalues of $H\Sigma$, $r = rank(\Sigma H \Sigma)$, and $R$ is the matrix of corresponding orthonormal eigenvectors. Then
\begin{align*}
	X^T H X  &\ed (SY)^T H SY \ed Y^T R^T\Lambda R Y \\
	&\ed Z^T \Lambda Z = \sum_{i=1}^r \beta_i Z_i^2,
\end{align*}
where $Z_i$ are independent standard Gaussian random variables. Therefore, we have the quadratic form of Gaussian variables
\[
\sqrt{n}(\theta - \hat{\theta}_n)^T H \sqrt{n}(\theta - \hat{\theta}_n) \ed \sum_{i = 1}^r \beta_i Z_i^2.
\]
This completes the proof.
\end{proof}

\textit{Remark}: Even though Bregman divergence is asymmetric, $n D_{\phi}(\hat{p}_n, p)$ has the same asymptotic distribution as $n D_{\phi}(p,\hat{p}_n)$ by a similar proof.

Noting that $\hat{p}_n$ is the maximum likelihood estimator of $p$, we immediately arrive at the following corollary.

\begin{corollary} \label{asym}
For a discrete distribution $p = (p_1, \ldots, p_d)$ and the empirical distribution $\hat{p}_n = (\hat{p}_{n,1}, \ldots, \hat{p}_{n,d})$ generated from $n$ iid samples, we have 
	\[
\lim_{n \to \infty} n D_{\phi}(p,\hat{p}_n) \convd \frac{1}{2} \sum_{i = 1}^r \beta_i Z_i^2,
	\]
where $Z_i$ are independent standard Gaussian random variables, $r = rank(\Sigma^T H \Sigma)$, $H$ is the Hessian of $\phi$, $\Sigma$ is the inverse Fisher information matrix, and $\beta_1, \ldots, \beta_r$ are the nonzero eigenvalues of $H\Sigma$.
\end{corollary}

\subsection{Wasserstein-Bregman Divergence}
In this section, we first define the new Wasserstein-Bregman divergence between probability distributions. We then show that under some mild conditions, the divergence function is differentiable with respect to the parameters almost everywhere. This result allows the gradient descent algorithm to minimize the divergence between a target distribution and a parametric distribution. Therefore, Wasserstein-Bregman divergence can be used as an objective function in distribution learning.

\begin{definition} Let $\phi: \RR^d \to \RR$ be a strictly convex function and $D_\phi: \RR^d \times \RR^d \to \RR$ be the associated Bregman divergence with $D_\phi(x,y) = \phi(x) - \phi(y) - \langle \nabla \phi(y), x-y\rangle$. \textit{Wasserstein-Bregman divergence} $W_{D_\phi} (\PP, \QQ)$ is defined as
\begin{align*}
W_{D_\phi} (\PP, \QQ) &= \inf_{\gamma \in \Pi(\PP,\QQ)} \int D_\phi(x,y) d\gamma(x,y) \\
&= \inf_{\gamma \in \Pi(\PP,\QQ)} \EE_{X,Y \sim \gamma} [D_\phi(X,Y)].
\end{align*}
\end{definition}

As an example, if $\phi(x) = \|x\|^2$, $W_{D_\phi} (\PP, \QQ)$ reduces to $W_2 (\PP, \QQ)^2$.

By the nonnegativity of Bregman divergence, it is easy to verify that $W_{D_\phi} (\PP, \QQ)$ is always nonnegative, and $W_{D_\phi} (\PP, \QQ)= 0$ if and only if $\PP = \QQ$.

\begin{theorem}
	Let $\mathcal{X} \subset \RR^d $ be a compact metric set, $\QQ$ is a fixed distribution defined on $\mathcal{X}$, $g_\theta(Z)$ is a function of $Z$, with parameter $\theta \in \RR^d$, and $Z$ being a random variable over another space $\mathcal{Z}$. Let $\PP_\theta$ denote the distribution of $g_\theta(Z)$. Then
\begin{enumerate}
	\item If $g$ is continuous in $\theta$, then $W_{D_\phi}(\QQ, \PP_\theta)$ is also continuous in $\theta$.
	\item If $g$ is locally Lipschitz with local Lipschitz constants $L(\theta,z)$ such that $\EE_{Z\sim \PP_\theta} L(\theta, Z)^2 < \infty$, then $W_{D_\phi}(\QQ, \PP_\theta)$ is differentiable almost everywhere.
\end{enumerate}
\end{theorem}
Remark: in this theorem, $g_\theta(Z)\sim \PP_\theta$ is the parametric distribution that attempts to replicate the distribution $\QQ$.
\begin{proof}
Because $\phi$ is strictly convex, its gradient $\nabla \phi$ has positive definite Jacobian matrix, which is also the Hessian of $\phi$. Then by the inverse function theorem, $\nabla \phi$ is invertible. Denote its inverse with $(\nabla \phi)^{-1}$ and the composition of $\PP_\theta$ and $(\nabla \phi)^{-1}$ as $\PP_\theta \circ (\nabla \phi)^{-1}$, then $\EE_{Y \sim \PP_\theta \circ (\nabla \phi)^{-1}} [||Y||_2^2] = \EE_{Y \sim \PP_\theta} [\|\nabla \phi(Y)\|_2^2]$.

Expand $W_{D_\phi} (\QQ, \PP_\theta)$ by the linearity of inner products,
\begin{align*}
	W_{D_\phi} &(\QQ, \PP_\theta)  \\
	&= \inf_{\gamma \in \Pi(\QQ,\PP_\theta)} \int [\phi(x) - \phi(y) - \langle \nabla \phi(y), x-y\rangle] d\gamma(x,y) \\
	&= \EE_{X \sim \QQ} [\phi(X)] - \EE_{Y \sim \PP_\theta} [\phi(Y)] + \EE_{Y \sim \PP_\theta} [\langle\phi(Y), Y \rangle] \\
	&\qquad+ \inf_{\gamma \in \Pi(\QQ,\PP_\theta)} \EE_{X,Y \sim \gamma} [\langle -\nabla \phi(Y), X \rangle] \\
	&= \frac{1}{2}[ \EE_{X \sim \QQ} [\|X\|_2^2] + \EE_{Y \sim \PP_\theta} [\|\nabla \phi(Y)\|_2^2] \\
	&\qquad + \inf_{\gamma \in \Pi(\QQ,\PP_\theta)} \EE_{X,Y \sim \gamma} [\langle -2\nabla \phi(Y), X \rangle]] \\
	&\qquad + \EE_{X \sim \QQ} [\phi(X)] - \EE_{Y \sim \PP_\theta} [\phi(Y)] \\
	&\qquad + \EE_{Y \sim \PP_\theta} [\langle\nabla\phi(Y), Y \rangle] \\
	&\qquad- \frac{1}{2}\left[\EE_{X \sim \QQ} [\|X\|_2^2] + \EE_{Y \sim \PP_\theta} [\|\nabla\phi(Y)\|_2^2]\right] \\
	&= \frac{1}{2} W_2(\QQ, \PP_\theta \circ (\nabla \phi)^{-1})^2 + \EE_{X \sim \QQ} [\phi(X)] \\
	&\qquad- \EE_{Y \sim \PP_\theta} [\phi(Y)] + \EE_{Y \sim \PP_\theta} [\langle\nabla\phi(Y), Y \rangle] \\
	&\qquad- \frac{1}{2}\left[\EE_{X \sim \QQ} [\|X\|_2^2] + \EE_{Y \sim \PP_\theta \circ (\nabla \phi)^{-1}} [||Y||_2^2]\right].
\end{align*}
Therefore we can express the new $W_{D_\phi}(\QQ, \PP_\theta)$ as the distorted squared Wasserstein distance $\frac{1}{2} W_2(\QQ, \PP_\theta \circ (\nabla \phi)^{-1})^2$ plus some error correction terms, which do not depend on the choice of coupling $\gamma$. 

Now it suffices to  show that $W_{L^2}(\bp_r,\bp_\theta)$ is almost everywhere differentiable. First, observe that for two vectors $\theta, \theta' \in \RR^d$, let $\pi$ be the joint distribution of $(g_\theta(Z), g_{\theta'}(Z))$ where $Z \sim p(z)$, then
  \begin{align*}
    W_{L^2}(\bp_\theta,\bp_\theta') &\le  \left[\be_{(X,Y) \sim \pi} [\|X-Y\|_2^2] \right]^{\frac{1}{2}} = \left[ \be \|g_\theta(Z) - g_{\theta'}(Z)\|_2^2 \right]^{\frac{1}{2}}.
  \end{align*}
  The continuity of $g_\theta$ ensures that $\|g_\theta(Z) - g_{\theta'}(Z)\|^2 \to 0 $ point-wise as $\theta \to \theta'$. Since $\XCal$ is compact, $\|g_\theta(Z) - g_{\theta'}(Z)\|^2$ is uniformly bounded. Therefore by the bounded convergence theorem,
  \[
    W_{L^2}(\bp_\theta,\bp_\theta') \le \left [\be \|g_\theta(Z) - g_{\theta'}(Z)\|_2^2 \right]^{\frac{1}{2}} \to 0, \text{ as } \theta \to \theta'.
  \]
  Hence by the triangle inquality, as $\theta \to \theta'$, $
    |W_{L^2}(\bp_r, \bp_\theta) - W_{L^2}(\bp_r, \bp_{\theta'})| \le  W_{L^2}(\bp_\theta, \bp_\theta') \to 0.
  $
  This proves the continuity.
  
  Now assume $g_\theta$ is locally Lipschitz, i.e., for each pair $(\theta,z)$, there exists a constant $L(\theta,z)$ and an open neighborhood $N(\theta,z)$ around $(\theta,z)$ such that $\forall (\theta',z') \in N(\theta,z)$,
  \[
    \|g_\theta(z) - g_{\theta'}(z')\|_2 \le L(\theta,z)(\|\theta - \theta'\|_2 + \|z - z'\|_2).
  \]
  By fixing $z' = z$ and taking expectation of squares of both sides, we get
  \[
    \be \|g_\theta(Z) - g_{\theta'}(Z)\|_2^2 \le \|\theta - \theta'\|_2^2 \be [L(\theta,Z)^2],
  \]
  for all $\theta'$ in an open neighborhood of $\theta$. Therefore,
  \[
    |W_{L^2}(\bp_r, \bp_\theta) - W_{L^2}(\bp_r, \bp_{\theta'})| \le  W_{L^2}(\bp_\theta, \bp_\theta') \le \left[\be \|g_\theta(Z) - g_{\theta'}(Z)\|_2\right]^{\frac{1}{2}} \le \|\theta - \theta'\|_2 \be [L(\theta, Z)^2]^{\frac{1}{2}},
  \]
  i.e., $W_{L^2}(\bp_r, \bp_\theta)$ is locally Lipschitz and by Rademacher's theorem, is differentiable almost everywhere.
\end{proof}

\section{Discussion}

\subsection{DRSO and Ambiguity Set}

Suppose one chooses the divergence between probability distributions to be $d(\PP,\QQ)$, where $\PP$ and $\QQ$ are probability measures defined on the set $\mathcal{X} \subset \RR^n$. Let $\mathcal{M}_+$ denotes the set of all probability distributions defined over the set $\mathcal{X}$. Then the ambiguity set $\mathcal{P}$ can be defined as a ball centered at the nominal distribution $\QQ$:
\[
\mathcal{P} = \{\PP \in \mathcal{M}_+: d(\PP,\QQ) \le \delta\}.
\]

The nominal distribution $\QQ$ may come from prior knowledge of the model, or directly from data. In the data-driven setting where we are given iid samples $\{X_i\}_{i=1}^n$, the nominal distribution $\QQ$ is chosen to be the empirical distribution $\hat{\PP}_n$. 

\begin{itemize}
\item When the sample size $n$ is large (relative to $d$), one can appeal to the asymptotic distribution of $D(p, \hat{p}_n)$ to construct an ambiguity set using Theorem \ref{asym}. More specifically, an ambiguity set can be constructed as follows:
$$\mathcal{P} = \{p : D_\phi(p, \hat{p}_n) \leq \frac{1}{2n}F^{-1}(\alpha)\},$$
where $F^{-1}(\alpha)$ is the quantile function of $\sum_{i = 1}^r \beta_i Z_i^2$, which is a weighted sum of independent $\chi^2$ random variables with one degree of freedom. This quantile can be approximated via a Monte Carlo approximation. For a large $K$ (say $K = 10000$), one can simulate $rK$ independent standard normal random variables $Z_{1, 1}, ..., Z_{1, r}, Z_{2, 1}, ..., Z_{2, r}, ..., Z_{K, 1}, ..., Z_{K, r},$ and compute $R_j = \sum_{i = 1}^r \beta_i Z_{i, j}^2$ for each $j = 1, ... K$. Then one can use take the $\alpha$-th empirical quantile of $(R_1, ..., R_K)$ as an approximation to $F^{-1}(\alpha)$. Note that $\mathcal{P}$ is convex since Bregman divergence is convex with respect to the first argument.
\item When the sample size $n$ is of moderate size or small, one must appeal to concentration results to obtain a valid ambiguity set. In order to apply Theorem \ref{nonconvexConcentration} or Theorem \ref{convexConcentration} for the construction of the ambiguity set, we have to first derive an upper bound for $\EE D_\phi(p, \hat{p}_{n})$ or $\EE D_\phi(\hat{p}_{n}, p)$, respectively.

For $\EE D_\phi(p, \hat{p}_{n})$, clearly
\begin{align*}
	\EE &D_\phi(p, \hat{p}_{n}) = \EE[\phi(p) - \phi(\hat{p}_{n}) - \langle \nabla\phi(\hat{p}_{n}), p - \hat{p}_{n}\rangle] \\
	&= \EE[\phi(p) - \phi(\hat{p}_{n})] - \EE[\langle \nabla\phi(p) - \nabla\phi(\hat{p}_{n}), p - \hat{p}_{n}\rangle] \\
	&\le M_{\phi}  \sqrt{\sum_{i=1}^d \frac{p_i(1-p_i)}{n}} + L_\phi \EE\|p - \hat{p}_n\|_2^2 \\
	&= M_{\phi}  \sqrt{\sum_{i=1}^d \frac{p_i(1-p_i)}{n}} + L_\phi \sum_{i=1}^d \frac{p_i(1-p_i)}{n}\\
	&= M_{\phi}   \sqrt{\frac{d}{4n}} + L_\phi \frac{d}{4n},
\end{align*}
where the inequality is by the Cauchy-Schwarz inequality and the Taylor's theorem.

Similarly, for $\EE D_\phi(\hat{p}_{n}, p)$,
\begin{align*}
	\EE \left[D_\phi(\hat{p}_{n}, p)\right] 
	&= \EE [ \phi(\hat{p}_{n}) - \phi(p)] \\
	&= \EE [\langle\nabla \phi(\xi), \hat{p} - p\rangle] \\
	&\le \EE [\|\nabla \phi(\xi)\|\|\hat{p} - p\|_2] \\
	&\le M_{\phi}  \EE [\|\hat{p} - p\|_2] \\
	&\le M_{\phi}  \sqrt{\EE [\|\hat{p} - p\|_2^2]} \\
	&= M_{\phi}  \sqrt{\sum_{i=1}^d \frac{p_i(1-p_i)}{n}}\\
	&\leq M_{\phi}   \sqrt{\frac{d}{4n}},
\end{align*}
where $\xi$ is between $\hat{p}$ and $p$, the first inequality is by Cauchy-Schwarz, and the third inequality is by the Jensen's inequality. 
\item As described immediately after the proof of Theorem \ref{nonconvexConcentration}, the resulting ambiguity set might be intractable to be computed because of its potential nonconvex nature. On the other hand, Theorem \ref{convexConcentration} 
results in a convex ambiguity set, which is the Bregman ball of $\hat{p_n}$ with radius $M_\phi \sqrt{\frac{d}{4n}} + L_\phi \left(\frac{d}{4n}\right) + \epsilon$:
$$\left\{p: D_\phi(p, \hat{p_n}) \leq M_{\phi}  \sqrt{\frac{d}{4n}} + L_\phi \left(\frac{d}{4n}\right) + \epsilon \right\}.$$
\end{itemize}

\subsection{Distribution Learning}

For distribution learning, the Wasserstein-Bregman divergence $W_{D_\phi}$ can be served as the objective function in the optimization problem. As shown in the proof of Theorem 3.4, the Wasserstein-Bregman divergence $W_{D_\phi}$ has an interesting decomposition in terms of squared divergence plus a penalty term: 
$$W_{D_\phi}(\QQ, \PP_\theta) = D + P,$$
where
\begin{align*}
	D &=\frac{1}{2} W_2(\QQ, \PP_\theta \circ (\nabla \phi)^{-1})^2 \\
	P &= \EE_{X \sim \QQ} [\phi(X)] - \EE_{Y \sim \PP_\theta} [\phi(Y)] + \EE_{Y \sim \PP_\theta} [\langle\nabla\phi(Y), Y \rangle] \\
	&\qquad- \frac{1}{2}\left[\EE_{X \sim \QQ} [\|X\|_2^2] + \EE_{Y \sim \PP_\theta \circ (\nabla \phi)^{-1}} [||Y||_2^2]\right].
\end{align*}

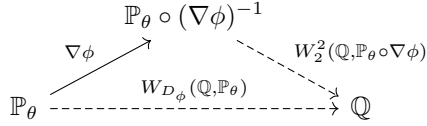
\begin{figure}
\[
\begin{tikzcd}
 & \PP_\theta \circ (\nabla \phi)^{-1} \arrow[dr, dashrightarrow]{d}{W_{2}^2(\QQ, \PP_\theta \circ \nabla \phi)} \\
 \PP_\theta \arrow{ur}{\nabla \phi} \arrow[rr, dashrightarrow]{d}{W_{D_\phi}(\QQ, \PP_\theta)} &&  \QQ
\end{tikzcd}
\]
\caption{Schematic diagram of the decomposition of $W_{D_\phi}$. The solid arrow denotes transformation. The dashed arrow denotes divergence measure.}
\label{fig:schematic}
\end{figure}

From a high-level perspective, optimizing the Wasserstein-Bregman divergence between two distributions is basically optimizing the $L_2$-Wasserstein distance between one of the distribution and a $\nabla \phi$-transformed distribution (the $D$ term), with a penalty term accounting for the influence of $\phi$ in the divergence measure (the $P$ term). See Figure \ref{fig:schematic}.

By considering Wasserstein-Bregman divergence, we retain the choice of choosing a symmetric measure (say, $D_\phi(x, y) = ||x - y||_2^2$) or choosing an asymmetric measure (say, $D_\phi(x , y) = \sum_{i = 1}^d x_i \log\left(\frac{x_i}{y_i}\right) - \sum_{i = 1}^d (x_i - y_i)$). In particular, in the special case when the Bregman divergence is chosen to be the $L_2$ distance, we get $L_2$-Wasserstein distance, as in \cite{Arjovsky:2017vh}. In contrast, all Wasserstein distances are always symmetric, since a metric is used within the definition.

\subsection{Future Work}

It remains an open problem on how to select the underlying convex function $\phi$ in Bregman divergence for a given problem. Proposition \ref{FisherInfo} provides some insight on how $\phi$ is selected. Based on the proposition, the amount of information containing in $p$ depends on the curvature of $\phi$. Choosing $\phi(x) = ||x||_2^2$ can be somewhat conservative in the sense that the amount of information is independent of the value of $p$. 

On the other hand, it remains to be further investigated as to the definite advantage of replacing the metric $d(x, y)$ in the Wasserstein distance by the Bregman divergence $D_{\phi}(x, y)$. Nevertheless, Wasserstein-Bregman divergence can serve as a viable candidate for measuring distributional divergence when asymmetry is desirable.


\newpage

\bibliography{paper.bib}

\begin{thebibliography}{36}
\providecommand{\natexlab}[1]{#1}
\providecommand{\url}[1]{\texttt{#1}}
\expandafter\ifx\csname urlstyle\endcsname\relax
  \providecommand{\doi}[1]{doi: #1}\else
  \providecommand{\doi}{doi: \begingroup \urlstyle{rm}\Url}\fi

\bibitem[Arjovsky et~al.(2017)Arjovsky, Chintala, and Bottou]{Arjovsky:2017vh}
M.~Arjovsky, S.~Chintala, and L.~Bottou.
\newblock {Wasserstein GAN}.
\newblock \emph{arXiv.org}, January 2017.

\bibitem[Banerjee et~al.(2005{\natexlab{a}})Banerjee, Guo, and
  Wang]{Banerjee:2005jd}
A.~Banerjee, X.~Guo, and H.~Wang.
\newblock {On the optimality of conditional expectation as a Bregman
  predictor}.
\newblock \emph{IEEE Transactions on Information Theory}, 51\penalty0
  (7):\penalty0 2664--2669, July 2005{\natexlab{a}}.

\bibitem[Banerjee et~al.(2005{\natexlab{b}})Banerjee, Merugu, Dhillon, and
  Ghosh]{Banerjee:2005vsa}
A.~Banerjee, S.~Merugu, I.~S. Dhillon, and J.~Ghosh.
\newblock {Clustering with Bregman divergences}.
\newblock \emph{Journal of Machine Learning Research}, 6\penalty0
  (Oct):\penalty0 1705--1749, 2005{\natexlab{b}}.

\bibitem[Bayraksan and Love(2015)]{Bayraksan:2015ge}
G.~Bayraksan and D.~K. Love.
\newblock {Data-driven stochastic programming using phi-divergences}.
\newblock \emph{Tutorials in Operations Research}, 2015.

\bibitem[Bregman(1967)]{Bregman:1967ab}
L.M. Bregman.
\newblock The relaxation method of finding the common point of convex sets and
  its application to the solution of problems in convex programming.
\newblock \emph{USSR Computational Mathematics and Physics}, \penalty0
  (7):\penalty0 200--217, 1967.

\bibitem[Buja et~al.(2005)Buja, Stuetzle, and Shen]{Buja05lossfunctions}
A.~Buja, W.~Stuetzle, and Y.~Shen.
\newblock Loss functions for binary class probability estimation and
  classification: structure and applications, 2005.
\newblock URL \url{www-stat.wharton.upenn.edu/~buja}.

\bibitem[Collins et~al.(2002)Collins, Schapire, and Singer]{Collins:2002wc}
M.~Collins, R.~E. Schapire, and Y.~Singer.
\newblock {Logistic regression, AdaBoost and Bregman distances}.
\newblock \emph{Machine Learning}, 2002.

\bibitem[Dobri\'{c} and Yukich(1995)]{Dobric307:1995}
V.~Dobri\'{c} and J.~E. Yukich.
\newblock {Asymptotics for transportation cost in high dimensions}.
\newblock \emph{Journal of Theoret. Probab.}, 8\penalty0 (1):\penalty0 97--118,
  1995.

\bibitem[Dobru\u{s}in(1970)]{Dobrusin308:1970}
R.~L. Dobru\u{s}in.
\newblock {Prescribing a system of random variables by conditional
  distributions}.
\newblock \emph{Theor. Prob. Appl.}, 15:\penalty0 458--486, 1970.

\bibitem[Dobru\u{s}in(1996)]{Dobrusin310:1996}
R.~L. Dobru\u{s}in.
\newblock {Perturbation methods of the theory of Gibbsian fields}.
\newblock \emph{Lectures on probability theory and statistics (Saint-Flour,
  1994)}, pages 1--66, 1996.

\bibitem[Esfahani and Kuhn(2015)]{Esfahani:2015wh}
P.~M. Esfahani and D.~Kuhn.
\newblock {Data-driven distributionally robust optimization using the
  Wasserstein metric: performance guarantees and tractable reformulations}.
\newblock \emph{arXiv.org}, 2015.

\bibitem[Fournier and Guillin(2015)]{Fournier:2015kk}
N.~Fournier and A.~Guillin.
\newblock {On the rate of convergence in Wasserstein distance of the empirical
  measure}.
\newblock \emph{Probability Theory and Related Fields}, 162\penalty0
  (3-4):\penalty0 707--738, 2015.

\bibitem[Gao and Kleywegt(2016)]{Gao:2016vo}
R.~Gao and A.~J. Kleywegt.
\newblock {Distributionally robust stochastic optimization with Wasserstein
  distance}.
\newblock \emph{arXiv.org}, April 2016.

\bibitem[Hu and Hong(2013)]{Hu:2013tc}
Z.~Hu and L.~J. Hong.
\newblock {Kullback-Leibler divergence constrained distributionally robust
  optimization}.
\newblock \emph{Available at Optimization Online}, 2013.

\bibitem[Jiang and Guan(2012)]{Guan:2012ec}
R.~Jiang and Y.~Guan.
\newblock {Data-driven chance constrained stochastic program}.
\newblock \emph{Mathematical Programming}, pages 1--37, 2012.

\bibitem[Jones and Byrne(1990{\natexlab{a}})]{Jones:1990tf}
L.~K. Jones and C.~L. Byrne.
\newblock {General entropy criteria for inverse problems, with applications to
  data compression, pattern classification, and cluster analysis}.
\newblock \emph{IEEE Transactions on Information Theory}, 1990{\natexlab{a}}.

\bibitem[Jones and Byrne(1990{\natexlab{b}})]{Jones:1990}
L.~K. Jones and C.L. Byrne.
\newblock {General entropy criteria for inverse problems, with applications to
  data compression, pattern classification, and cluster analysis}.
\newblock \emph{IEEE Transactions on Information Theory}, 36:\penalty0 23--30,
  1990{\natexlab{b}}.

\bibitem[Kivinen and Warmuth(2001)]{Kivinen:2001}
J.~Kivinen and M.~Warmuth.
\newblock {Relative loss bounds for multidimensional regression problems}.
\newblock \emph{Machine Learning}, 45\penalty0 (3):\penalty0 301--329, 2001.

\bibitem[Lafferty(1999)]{Lafferty:1999}
J.~Lafferty.
\newblock {Additive models, boosting, and inference for generalized
  divergences}.
\newblock \emph{Proceedings of Conference on Computational Learning Theory},
  pages 125--133, 1999.

\bibitem[LeBesenerais and Demoment(1999)]{LeBesenerais:1999}
G.~LeBesenerais and G.~Demoment.
\newblock {A new look at entropy for solving linear inverse problems}.
\newblock \emph{IEEE Transactions on Information Theory}, 45\penalty0
  (5):\penalty0 1565--1577, 1999.

\bibitem[Lucic et~al.(2016)Lucic, Bachem, and Krause]{Lucic:2016uv}
M.~Lucic, O.~Bachem, and A.~Krause.
\newblock {Strong coresets for hard and soft Bregman clustering with
  applications to exponential family mixtures}.
\newblock \emph{Proceedings of the 19th International Conference on Artificial
  Intelligence and Statistics}, 2016.

\bibitem[Murata et~al.(2004)Murata, Takenouchi, Kanamori, and
  Eguchi]{Murata:2004}
N.~Murata, T.~Takenouchi, T.~Kanamori, and S.~Eguchi.
\newblock {Information geometry of U-Boost and Bregman divergence}.
\newblock \emph{Neural Computation}, 16:\penalty0 1437--1481, 2004.

\bibitem[Nemirovski and D.(1983)]{Nemirovski:1983ab}
A.~Nemirovski and Yudin D.
\newblock \emph{Problem complexity and method efficiency in optimization}.
\newblock Wiley, 1983.

\bibitem[Pardo and Vajda(1997)]{Pardo:1999}
M.~C. Pardo and I.~Vajda.
\newblock {About distances of discrete distributions satisfying the data
  processing theorem of information theory}.
\newblock \emph{IEEE Transactions on Information Theory}, 43\penalty0
  (4):\penalty0 1288--1293, 1997.

\bibitem[Pardo and Vajda(2003)]{Pardo:2003fr}
M.C. Pardo and I.~Vajda.
\newblock {On asymptotic properties of information-theoretic divergences}.
\newblock \emph{IEEE Transactions on Information Theory}, 49\penalty0
  (7):\penalty0 1860--1868, 2003.

\bibitem[Peres(2005)]{Peres679:un}
Y.~Peres.
\newblock Mixing for markov chains and spin systems.
\newblock \emph{Unpublished notes}, 2005.
\newblock URL \url{www.stat.berkeley.edu/~peres/ubc.pdf}.

\bibitem[Rachev(1991)]{Rachev695:1991}
S.~T. Rachev.
\newblock \emph{{Probability metrics and the stability of stochastic models}}.
\newblock John Wiley \& Sons Ltd., Chichester, 1991.

\bibitem[Rachev and R\"{u}schendorf(1998)]{Rachev696:1998}
S.~T. Rachev and L.~R\"{u}schendorf.
\newblock \emph{{Mass transportation problems. Vol. I: Theory, Vol. II:
  Applications.}}
\newblock Springer-Verlag, New York, 1998.

\bibitem[Shafieezadeh-Abadeh et~al.(2015)Shafieezadeh-Abadeh, Esfahani, and
  Kuhn]{Shafieezadeh-Abadeh:2015wr}
S.~Shafieezadeh-Abadeh, P.~M. Esfahani, and D.~Kuhn.
\newblock {Distributionally robust logistic regression}.
\newblock \emph{arXiv.org}, September 2015.

\bibitem[Spohn(1991)]{Spohn757:1991}
H.~Spohn.
\newblock \emph{{Large scale dynamics of interacting particles}}.
\newblock Texts and Monographs in Physics. Springer-Verlag, Berlin, 1991.

\bibitem[Srivastava et~al.(2007)Srivastava, Gupta, and
  Frigyik]{Srivastava:2007}
S.~Srivastava, M.~R. Gupta, and B.~A. Frigyik.
\newblock {Bayesian quadratic discriminant analysis}.
\newblock \emph{Journal of Machine Learning Research}, 8:\penalty0 1287--1314,
  2007.

\bibitem[Talagrand(1995)]{Talagrand772:1995}
M.~Talagrand.
\newblock {Concentration of measure and isoperimetric inequalities in product
  spaces}.
\newblock \emph{Inst. Hautes \'{E}tudes Sci. Publ. Math.}, 81:\penalty0
  73--205, 1995.

\bibitem[Taskar et~al.(2006)Taskar, Lacoste-Julien, and Jordan]{Taskar:2006}
B.~Taskar, S.~Lacoste-Julien, and M.~I. Jordan.
\newblock {Structured prediction, dual extragradient and Bregman projections}.
\newblock \emph{Journal of Machine Learning Research}, 7:\penalty0 1627--1653,
  2006.

\bibitem[Villani(2009)]{Villani:2009ab}
C.~Villani.
\newblock \emph{{Optimal transport}}.
\newblock Springer-Verlag Berlin Heidelberg, 2009.

\bibitem[Wibisono and Wilson(2016)]{Wibisono:2016gx}
A.~Wibisono and A.~C. Wilson.
\newblock {A variational perspective on accelerated methods in optimization}.
\newblock In \emph{Proceedings of the National Academy of Sciences of the
  United States of America}, 2016.

\bibitem[Wozabal(2012)]{Anonymous:W0HO4cob}
D.~Wozabal.
\newblock {A framework for optimization under ambiguity}.
\newblock \emph{Annals of Operations Research}, 2012.

\end{thebibliography}
\bibliographystyle{plainnat}
\end{document}